\documentclass[11pt|a4paper]{article}
\usepackage{acl2015}
\usepackage{times}
\pdfoutput=1
\usepackage{url}
\usepackage{multirow}
\usepackage{amsmath,amsthm,amsfonts}
\usepackage{graphicx,caption,subcaption}
\usepackage{float}
\usepackage{array}
\usepackage{algorithm}
\usepackage{algorithmic}

\DeclareMathOperator{\Q}{\textsc{Q}}
\DeclareMathOperator{\Op}{\odot}
\DeclareMathOperator{\Val}{\textsc{Val}}
\DeclareMathOperator{\Pair}{\textsc{Pair}}
\DeclareMathOperator{\Rel}{\textsc{Irr}}

\DeclareMathOperator{\I}{\mathcal{I}}
\DeclareMathOperator{\T}{\mathcal{T}}

\DeclareMathOperator{\Score}{\mathrm{Score}}

\newcolumntype{L}[1]{>{\raggedright\let\newline\\\arraybackslash\hspace{0pt}}m{#1}}
\newcolumntype{C}[1]{>{\centering\let\newline\\\arraybackslash\hspace{0pt}}m{#1}}
\newcolumntype{R}[1]{>{\raggedleft\let\newline\\\arraybackslash\hspace{0pt}}m{#1}}

\newenvironment{definition}[1][Definition]{\begin{trivlist}
\item[\hskip \labelsep {\bfseries #1}]}{\end{trivlist}}
\newtheorem{theorem}{Theorem}[section]
\title{Solving General Arithmetic Word Problems}

\author{Subhro Roy \quad\textnormal{and}\quad Dan Roth \\
  University of Illinois, Urbana Champaign \\
  {\tt \{sroy9, danr\}@illinois.edu}}

\date{}

\begin{document}
\maketitle
\begin{abstract}
  This paper presents a novel approach to automatically solving
  arithmetic word problems. This is the first algorithmic approach
  that can handle arithmetic problems with multiple steps and
  operations, without depending on additional annotations or
  predefined templates.  We develop a theory for expression trees that
  can be used to represent and evaluate the target arithmetic
  expressions; we use it to uniquely decompose the target arithmetic
  problem to multiple classification problems; we then compose an
  expression tree, combining these with world knowledge through a
  constrained inference framework. Our classifiers gain from the use
  of {\em quantity schemas} that supports better extraction of
  features. Experimental results show that our method outperforms
  existing systems, achieving state of the art performance on
  benchmark datasets of arithmetic word problems.

\end{abstract}

\section{Introduction}
\label{sec:intro}
  
  In recent years there is growing interest in understanding natural
  language text for the purpose of answering science related questions
  from text as well as quantitative problems of various kinds.  In
  this context, understanding and solving arithmetic word problems is
  of specific interest. Word problems arise naturally when reading
  the financial section of a newspaper, following election coverage,
  or when studying elementary school arithmetic word problems.  These
  problems pose an interesting challenge to the NLP community, due to
  its concise and relatively straightforward text, and seemingly
  simple semantics. Arithmetic word problems are usually directed
  towards elementary school students, and can be solved by combining
  the numbers mentioned in text with basic operations (addition,
  subtraction, multiplication, division). They are simpler than
  algebra word problems which require students to identify variables,
  and form equations with these variables to solve the problem.

  Initial methods to address arithmetic word problems have mostly
  focussed on subsets of problems, restricting the number or the type
  of operations used \cite{RoyViRo15,HosseiniHaEt14} but could not
  deal with multi-step arithmetic problems involving all four basic
  operations.  The template based method of \cite{KushmanZeBa14}, on
  the other hand, can deal with all types of problems, but implicitly
  assumes that the solution is generated from a set of predefined
  equation templates.

  In this paper, we present a novel approach which can solve a general
  class of arithmetic problems without predefined equation
  templates. In particular, it can handle multiple step arithmetic
  problems as shown in Example $1$.

  \begin{table}[H]
   \centering \small
   \begin{tabular}{|p{7cm}|}
     \hline Example 1 \\
     \hline {\em Gwen was organizing
       her book case making sure each of the shelves had exactly $9$
       books on it. She has $2$ types of books - mystery books
       and picture books. If she had $3$ shelves of mystery books and
       $5$ shelves of picture books, how many books did she have in
       total?} \\
     \hline
   \end{tabular}
   \label{tab:example1}
  \end{table}

  The solution involves understanding that the number of shelves needs
  to be summed up, and that the total number of shelves needs to be
  multiplied by the number of books each shelf can hold. In addition,
  one has to understand that the number ``2'' is not a direct part of
  the solution of the problem.

  While a solution to these problems eventually requires composing
  multi-step numeric expressions from text, we believe that directly
  predicting this complex expression from text is not feasible.

  At the heart of our technical approach is the novel notion of an
  {\em Expression Tree}. We show that the arithmetic expressions we are
  interested in can always be represented using an Expression Tree
  that has some unique decomposition properties. This allows us to
  decompose the problem of mapping the text to the arithmetic
  expression to a collection of simple prediction problems, each
  determining the lowest common ancestor operation between a pair of
  quantities mentioned in the problem. We then formulate the decision
  problem of composing the final expression tree as a joint inference
  problem, via an objective function that consists of all these
  decomposed prediction problems, along with legitimacy and background
  knowledge constraints.
  
  Learning to generate the simpler decomposed expressions allows us to support
  generalization across problems types. In particular, our system
  could solve Example $1$ even though it has never seen a problem that
  requires both addition and multiplication operations.

  We also introduce a second concept, that of {\em quantity schema},
  that allows us to focus on the information relevant to each quantity
  mentioned in the text. We show that features extracted from quantity
  schemas help reasoning effectively about the solution. Moreover,
  quantity schemas help identify unnecessary text snippets in the
  problem text.  For instance, in Example $2$, the information that
  ``Tom washed cars over the weekend'' is irrelevant; he could have
  performed any activity to earn money. In order to solve the problem,
  we only need to know that he had \$$76$ last week, and now he
  has \$$86$.

  \setlength{\tabcolsep}{6pt}
  \begin{table}[H]
   \centering \small
   \begin{tabular}{|p{7cm}|}
     \hline Example 2 \\
     \hline {\em Last week Tom had \$$74$. He washed cars
     over the weekend and now has \$$86$. How much money did he make
     from the job?} \\
     \hline
   \end{tabular}
   \label{tab:example2}
  \end{table}

  We combine the classifiers' decisions using a constrained inference
  framework that allows for incorporating world knowledge as
  constraints. For example, we deliberatively incorporate the
  information that, if the problems asks about an ``amount'', the
  answer must be positive, and if the question starts with ``how
  many'', the answer will most likely be an integer.

  Our system is evaluated on two existing datasets of arithmetic word
  problems, achieving state of the art performance on both. We also
  create a new dataset of multistep arithmetic problems, and show that
  our system achieves competitive performance in this challenging
  evaluation setting.

  The next section describes the related work in the area of automated
  math word problem solving. We then present the theory of expression
  trees and our decomposition strategy that is based on it. Sec. 4
  presents the overall computational approach, including the way we
  use quantity schemas to learn the mapping from text to expression tree
  components. Finally, we discuss our experimental study and conclude.

\section{Related Work}
\label{sec:related}
  Previous work in automated arithmetic problem solvers has focussed
  on a restricted subset of problems. The system described
  in \cite{HosseiniHaEt14} handles only addition and subtraction
  problems, and requires additional annotated data for verb
  categories. In contrast, our system does not require any additional
  annotations and can handle a more general category of problems.  The
  approach in \cite{RoyViRo15} supports all four basic operations, and
  uses a pipeline of classifiers to predict different properties of
  the problem. However, it makes assumptions on the number of
  quantities mentioned in the problem text, as well as the number of
  arithmetic steps required to solve the problem. In contrast, our
  system does not have any such restrictions, effectively handling
  problems mentioning multiple quantities and requiring multiple
  steps. Kushman's approach to automatically solving algebra word
  problems \cite{KushmanZeBa14} might be the most related to ours. It
  tries to map numbers from the problem text to predefined equation
  templates. However, they implicitly assume that similar equation
  forms have been seen in the training data. In contrast, our system
  can perform competitively, even when it has never seen similar
  expressions in training.

  There is a recent interest in understanding text for the purpose of
  solving scientific and quantitative problems of various kinds. Our
  approach is related to work in understanding and solving elementary
  school standardized tests \cite{Clark15}. The system described
  in \cite{BSCLHHCM14} attempts to automatically answer biology
  questions, by extracting the structure of biological processes from
  text. There has also been efforts to solve geometry questions by
  jointly understanding diagrams and associated text \cite{SHFE14}. A
  recent work \cite{SadeghiKuFa15} tries to answer science questions
  by visually verifying relations from images.

  Our constrained inference module falls under the general framework
  of Constrained Conditional Models (CCM) \cite{ChangRaRo12}. In
  particular, we use the $L+I$ scheme of CCMs, which predicts
  structured output by independently learning several simple
  components, combining them at inference time. This has been
  successfully used to incorporate world knowledge at inference time,
  as well as getting around the need for large amounts of jointly
  annotated data for structured
  prediction \cite{RothYi05,PunyakanokRoYi05,PunyakanokRoYi08,ClarkeLa06,BarzilayLa06,RoyViRo15}.

\section{Expression Tree and Problem Decomposition}
\label{sec:decomposition}

  We address the problem of automatically solving arithmetic word
  problems. The input to our system is the problem text $P$, which
  mentions $n$ quantities $q_1, q_2, \ldots, q_n$. Our goal is to map
  this problem to a read-once arithmetic expression $E$ that, when
  evaluated, provides the problem's solution. We define a read-once
  arithmetic expression as one that makes use of each quantity at most
  once. We say that $E$ is a {\em valid} expression, if it is such a
  Read-Once arithmetic expression, and we only consider in this work
  problems that {\em can be} solved using valid expressions (it's
  possible that they can be solved also with invalid expressions).


    An expression tree $\T$ for a valid expression $E$ is a binary
    tree whose leaves represent quantities, and each internal node
    represents one of the four basic operations. For a non-leaf node
    $n$, we represent the operation associated with it as $\Op(n)$,
    and its left and right child as $lc(n)$ and $rc(n)$
    respectively. The numeric value of the quantity associated with a
    leaf node $n$ is denoted as $\Q(n)$. Each node $n$ also has a value
    associated with it, represented as $\Val(n)$, which can be
    computed in a recursive way as follows:
    \begin{multline}\label{expr:formula}
      \Val(n) = \\
      \begin{cases}
        \Q(n) & \mbox{if $n$ is a leaf} \\
        \Val(lc(n)) \Op(n) \Val(rc(n)) &  \mbox{otherwise }
      \end{cases}
    \end{multline}
    For any expression tree $\T$ for expression $E$ with root node
    $n_{root}$, the value of $\Val(n_{root})$ is exactly equal to the
    numeric value of the expression $E$. Therefore,
    this gives a natural representation of numeric expressions,
    providing a natural parenthesization of the numeric
    expression. Fig \ref{fig:3} shows an example of an arithmetic
    problem with solution expression and an expression tree for the
    solution expression.
    
    \begin{figure}[ht]
    \centering \small
    \begin{tabular}{|p{3.5cm}|p{3.5cm}|}
    \hline \multicolumn{2}{|p{7cm}|}{Problem} \\
    \hline \multicolumn{2}{|p{7cm}|}{ 
       {\em Gwen was organizing her book case making sure each of the
       shelves had exactly $9$ books on it. She has $2$ types of books
       - mystery books and picture books. If she had $3$ shelves of
       mystery books and $5$ shelves of picture books, how many books
       did she have total?}} \\
     \hline\hline Solution & Expression Tree of Solution\\
     \hline \begin{center}$(3+5)\times9 = 72$\end{center} &
     \begin{center}
     \includegraphics[width=0.15\textwidth]{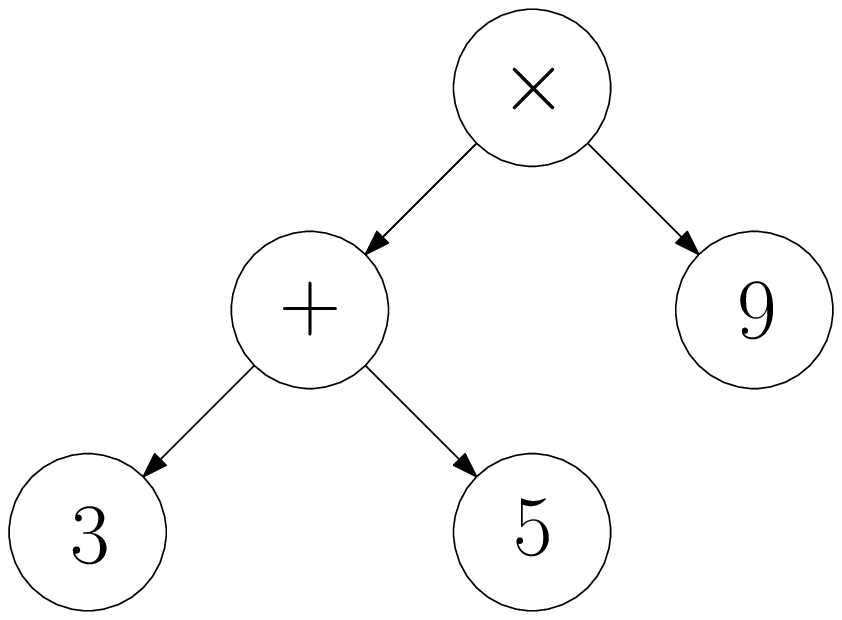}
     \end{center}\\
     \hline
    \end{tabular}
    \caption{\footnotesize An arithmetic word problem, solution expression
    and the corresponding expression tree}
    \label{fig:3}
    \end{figure}

    \begin{definition}
      An expression tree $\T$ for a valid expression $E$ is called 
      \textbf{\em monotonic} if it satisfies the following conditions:
      
      \begin{enumerate}

        \item If an addition node is connected to a
        subtraction node, then the subtraction node is the parent.

        \item If a multiplication node is connected to a division node, then
        the division node is the parent.

        \item Two subtraction nodes cannot be connected to each other.

        \item Two division nodes cannot be connected to each other.

      \end{enumerate}
        
    \end{definition}
    
    Fig \ref{fig:twotrees} shows two different expression trees for the same
    expression.  Fig \ref{fig:tree2} is monotonic whereas
    fig \ref{fig:tree1} is not.
    
    \begin{figure}[ht]
    \centering
        \begin{subfigure}[b]{0.46\linewidth}
        \centering
        \includegraphics[width=\linewidth]{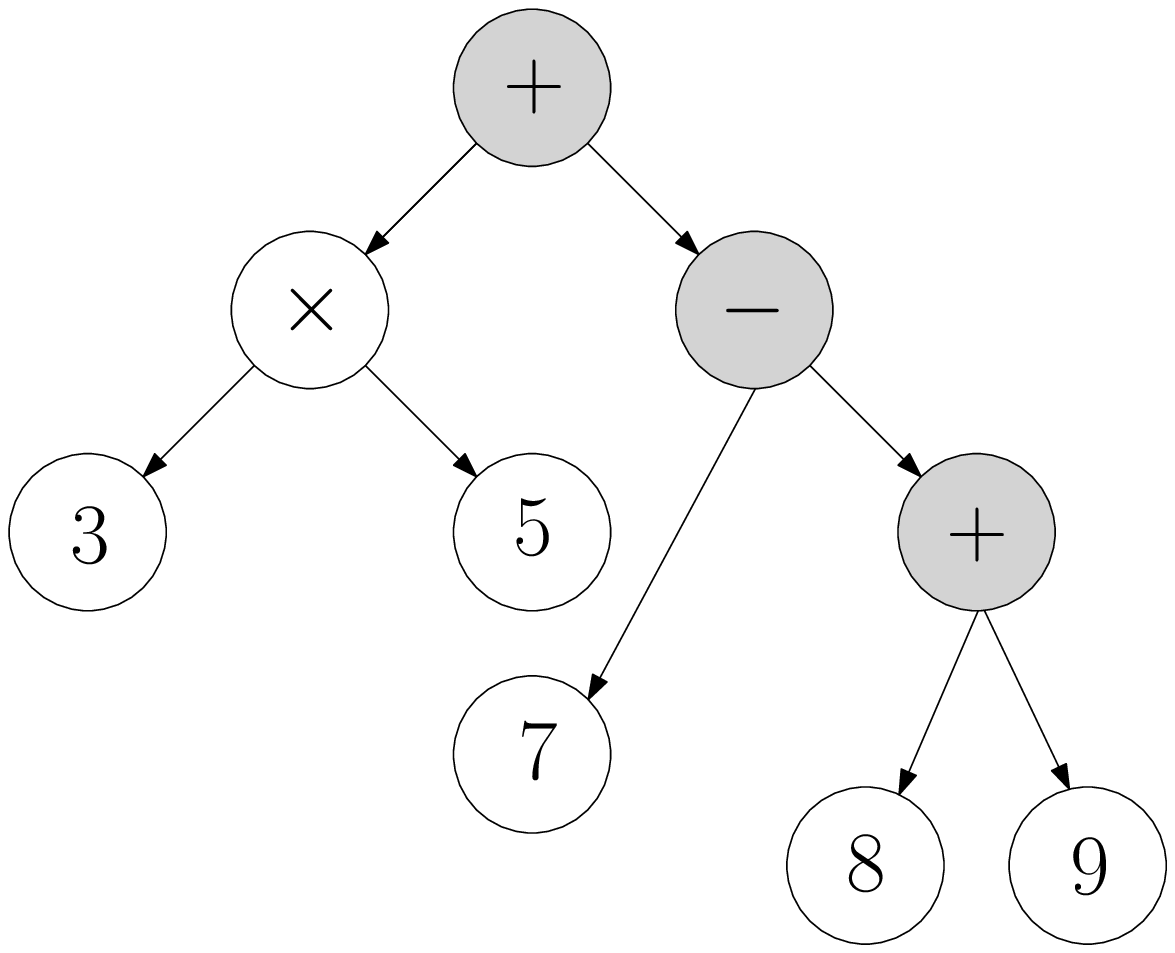}
        \caption{}
        \label{fig:tree1}
        \end{subfigure}%
        \quad
        \begin{subfigure}[b]{0.46\linewidth}
        \centering
        \includegraphics[width=\linewidth]{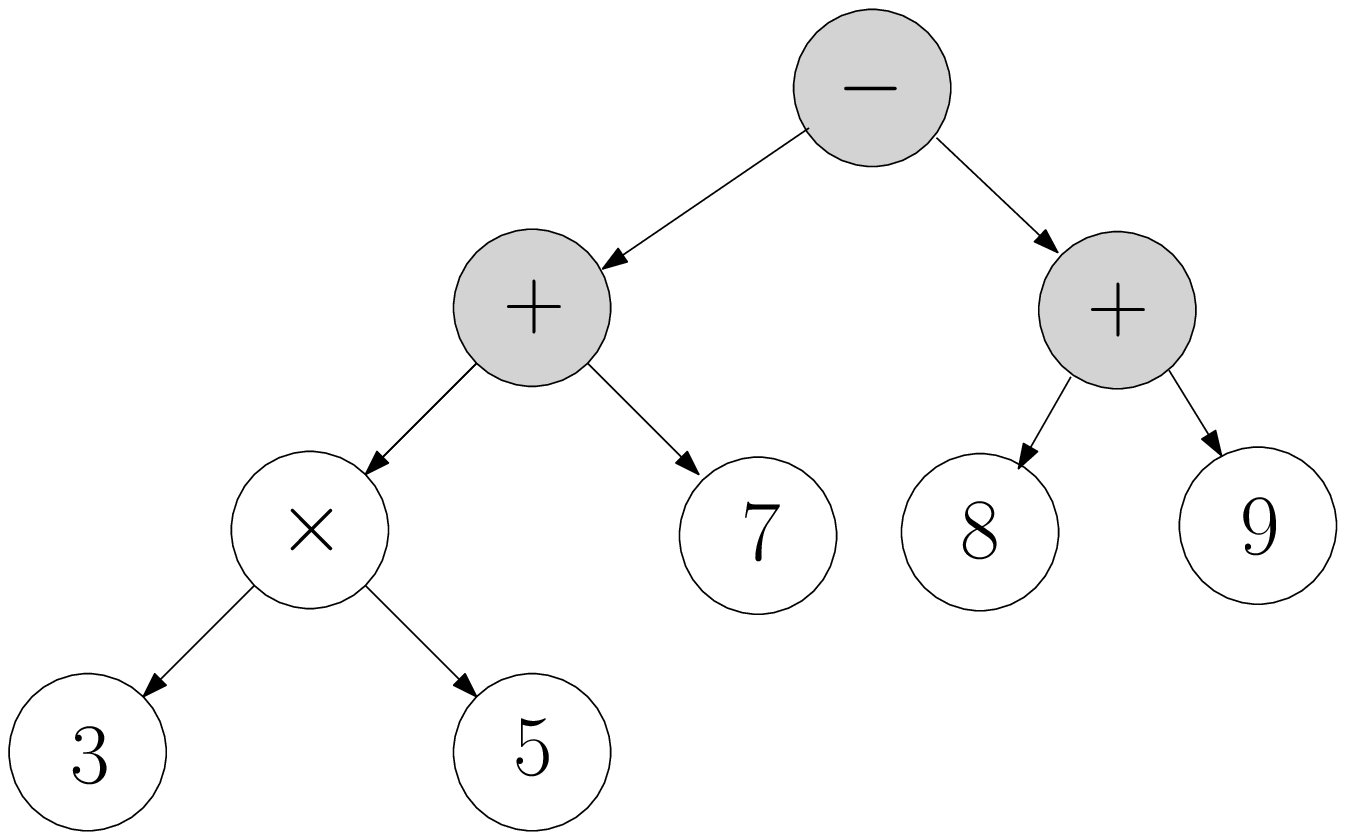}
        \caption{}
        \label{fig:tree2}
        \end{subfigure}
        \caption{\footnotesize Two different expression trees for the
        numeric expression $(3 \times 5)+7-8-9$. The right one is
        monotonic, whereas the left one is not. }
        \label{fig:twotrees}
    \end{figure}    

    Our decomposition relies on the idea of monotonic expression
    trees.  We try to predict for each pair of quantities $q_i,q_j$,
    the operation at the lowest common ancestor (LCA) node of the
    monotonic expression tree for the solution expression. We also
    predict for each quantity, whether it is relevant to the solution.
    Finally, an inference module combines all these predictions.

    In the rest of the section, we show that for any pair of
    quantities $q_i, q_j$ in the solution expression, any monotonic
    tree for the solution expression has the same LCA
    operation. Therefore, predicting the LCA operation becomes a
    multiclass classification problem.
    
    The reason that we consider the monotonic representation of the
    expression tree is that different trees could otherwise give
    different LCA operation for a given pair of quantities. For
    example, in Fig \ref{fig:twotrees}, the LCA operation for
    quantities $5$ and $8$ can be $+$ or $-$, depending on which tree
    is considered.

    \begin{definition}

      We define an \textbf{\em addition-subtraction chain} of an
      expression tree to be the maximal connected set of nodes labeled
      with addition or subtraction.

      The nodes of an addition-subtraction (AS) chain $C$ represent a
      set of terms being added or subtracted. These terms are
      sub-expressions created by subtrees rooted at neighboring nodes
      of the chain. We call these terms the \textbf{\em chain
      terms} of $C$, and the whole expression, after node operations
      have been applied to the chain terms, the \textbf{\em chain
      expression} of $C$. For example, in fig \ref{fig:twotrees}, the
      shaded nodes form an addition-subtraction chain. The chain
      expression is $(3 \times 5)+7 - 8 - 9$, and the chain terms are
      $3 \times 5$, $7$, $8$ and $9$. We define a \textbf{\em
      multiplication-division (MD) chain} in a similar way.

    \end{definition}
      
    \begin{theorem}

      Every valid expression can be represented by a
      monotonic expression tree.

    \end{theorem}

    \begin{proof}

      The proof is procedural, that is, we provide a method to convert
      any expression tree to a monotonic expression tree for the same
      expression. Consider a non-monotonic expression tree $E$, and
      without loss of generality, assume that the first condition for
      monotonicity is not valid. Therefore, there exists an addition
      node $n_i$ and a subtraction node $n_j$, and $n_i$ is the parent
      of $n_j$. Consider an addition-subtraction chain $C$ which
      includes $n_i, n_j$. We now replace the nodes of $C$ and its
      subtrees in the following way. We add a single subtraction node
      $n_-$. The left subtree of $n_-$ has all the addition chain
      terms connected by addition nodes, and the right subtree of
      $n_-$ has all the subtraction chain terms connected by addition
      nodes. Both subtrees of $n_-$ only require addition nodes, hence
      monotonicity condition is satisfied. We can construct the
      monotonic tree in Fig \ref{fig:tree2} from the non-monotonic
      tree of Fig \ref{fig:tree1} using this procedure. The addition
      chain terms are $3 \times 5$ and $7$, and the subtraction chain
      terms are $8$ and $9$. As as was described above, we introduce
      the root subtraction node in Fig \ref{fig:tree2} and attach the
      addition chain terms to the left and the subtraction chain terms
      to the right. The same line of reasoning can be used to handle
      the second condition with multiplication and division replacing
      addition and subtraction, respectively.

    \end{proof}

    \begin{theorem}\label{th:chain}

      Consider two valid expression trees $\T_1$ and $\T_2$ for the
      same expression $E$. Let $C_1$, $C_2$ be the chain containing
      the root nodes of $\T_1$ and $T_2$ respectively. The chain type
      (addition-subtraction or multiplication-division) as well as the
      the set of chain terms of $C_1$ and $C_2$ are identical.

    \end{theorem}

    \begin{proof} 

      We first prove that the chains containing the roots are both AS
      or both MD, and then show that the chain terms are also
      identical. 

      We prove by contradiction that the chain type is same. Let
      $C_1$'s type be ``addition-subtraction'' and $C_2$'s type be
      ``multiplication-division'' (without loss of generality). Since
      both $C_1$ and $C_2$ generate the same expression $E$, we have
      that $E$ can be represented as sum (or difference) of two
      expressions as well as product(or division) of two expressions.
      Transforming a sum (or difference) of expressions to a product
      (or division) requires taking common terms from the
      expressions, which imply that the sum (or difference) had
      duplicate quantities. The opposite transformation adds same term
      to various expressions leading to multiple uses of the same
      quantity. Therefore, this will force at least one of $C_1$ and
      $C_2$ to use the same quantity more than once, violating
      validity.

      We now need to show that individual chain terms are also
      identical. Without loss of generality, let us assume that both
      $C_1$ and $C_2$ are ``addition-subtraction'' chains. Suppose the
      chain terms of $C_1$ and $C_2$ are not identical. The chain
      expression for both the chains will be the same (since they are
      root chains, the chain expressions has to be the same as
      $E$). Let the chain expression for $C_1$ be $\sum_{i} t_i
      - \sum_i t'_i$, where $t_i$'s are the addition chain terms and
      $t'_i$ are the subtraction chain terms. Similarly, let the chain
      expression for $C_2$ be $\sum_{i} s_i - \sum_i s'_i$. We know
      that $\sum_{i} t_i - \sum_i t'_i = \sum_{i} s_i - \sum_i s'_i $,
      but the set of $t_i$'s and $t'_i$'s is not the same as the set
      of $s_i$ and $s'_i$'s. However it should be possible to
      transform one form to the other using mathematical
      manipulations. This transformation will involve taking
      common terms, or multiplying two terms, or both.  Following
      previous explanation, this will force one of the expressions to
      have duplicate quantities, violating validity. Hence, the chain
      terms of $C_1$ and $C_2$ are identical.
      
    \end{proof}
    
    Consider an expression tree $\T$ for a valid expression
    $E$.  For a distinct pair of quantities $q_i, q_j$ participating
    in expression $E$, we denote by $n_i, n_j$ the leaves of the
    expression tree $\T$ representing $q_i, q_j,$ respectively. Let
    $n_{LCA}(q_i, q_j; \T)$ to be the lowest common ancestor node of
    $n_i$ and $n_j$. We also define $order(q_i,q_j;\T)$ to be true if
    $n_i$ appears in the left subtree of $n_{LCA}(q_i, q_j;\T)$ and
    $n_j$ appears in the right subtree of $n_{LCA}(q_i, q_j; \T)$ and
    set $order(q_i,q_j;\T)$ to false otherwise. Finally we define
    $\Op_{LCA}(q_i,q_j;\T)$ for a pair of quantities $q_i, q_j$ as
    follows :
    \begin{multline}\label{expr:lca}
      \Op_{LCA}(q_i,q_j,\T) = \\
      \begin{cases}
        + & \mbox{ if }  \Op(n_{LCA}(q_i,q_j;\T))=+ \\ 
        \times & \mbox{ if } \Op(n_{LCA}(q_i,q_j;\T))=\times \\
        - & \mbox{ if } \Op(n_{LCA}(q_i,q_j;\T))=- \text{ and } \\
        & order(q_i,q_j;\T)=true \\
        -_{reverse} & \mbox{ if } \Op(n_{LCA}(q_i,q_j;\T))=- \text{ and } \\
        & order(q_i,q_j;\T)=false \\
        \div & \mbox{ if } \Op(n_{LCA}(q_i,q_j;\T))=\div \text{ and } \\
        & order(q_i,q_j;\T)=true \\
        \div_{reverse} & \mbox{ if } \Op(n_{LCA}(q_i,q_j;\T))=\div \text{ and } \\
        & order(q_i,q_j;\T)=false \\
      \end{cases}
    \end{multline}

    \begin{definition}

      Given two expression trees $\T_1$ and $\T_2$ for the same
      expression $E$, $\T_1$ is {\em LCA-equivalent} to $\T_2$ if for
      every pair quantities $q_i,q_j$ in the expression $E$, we have
      $\Op_{LCA}(q_i,q_j,\T_1) = \Op_{LCA}(q_i,q_j,\T_2)$.

    \end{definition}
    
    \begin{theorem}
    
      All monotonic expression trees for an expression are
      LCA-equivalent to each other.
    
    \end{theorem}
    
    \begin{proof}
      
      We prove by induction on the number of quantities used in an
      expression.  For all expressions $E$ with $2$ quantities, there
      exists only one monotonic expression tree, and hence, the
      statement is trivially true.  This satisfies our base case.
      
      For the inductive case, we assume that for all expressions with
      $k<n$ quantities, the theorem is true. Now, we need to prove
      that any expression with $n$ nodes will also satisfy the
      property.

      Consider a valid (as in all cases) expression $E$, with
      monotonic expression trees $\T_1$ and $\T_2$. From
      theorem \ref{th:chain}, we know that the chains containing the
      roots of $\T_1$ and $\T_2$ have identical type and terms. Given
      two quantities $q_i, q_j$ of $E$, the lowest common ancestor of
      both $\T_1$ and $\T_2$ will either both belong to the chain
      containing the root, or both belong to one of the chain
      terms. If the LCA node is part of the chain for both $\T_1$ and
      $\T_2$, monotonic property ensures that the LCA operation will
      be identical. If the LCA node is part of a chain term (which is
      an expression tree of size less than $n$), the property is
      satisfied by induction hypothesis.
      
    \end{proof}    

   The theory just presented suggests that it is possible to uniquely
   decompose the overall problem to simpler steps and this will be
   exploited in the next section.

\section{Mapping Problems to Expression Trees}
\label{sec:method}

  Given the uniqueness properties proved in
  Sec.~\ref{sec:decomposition}, it is sufficient to identify the
  operation between any two relevant quantities in the text, in order
  to determine the unique valid expression. In fact, identifying the
  operation between {\em any} pair of quantities provides much needed
  redundancy given the uncertainty in identifying the operation from
  text, and we exploit it in our final joint inference.
  
  Consequently, our overall method proceeds as follows: given the
  problem text $P$, we detect quantities $q_1, q_2, \ldots, q_n$. We
  then use two classifiers, one for relevance and other to predict the
  LCA operations for a monotonic expression tree of the solution. Our
  training makes use of the notion of quantity schemas, which we
  describe in Section \ref{method:quantschema}. The distributional
  output of these classifiers is then used in a joint inference
  procedure that determines the final expression tree.

  Our training data consists of problem text paired with a monotonic
  expression tree for the solution expression and alignment of
  quantities in the expression to quantity mentions in the problem
  text. Both the relevance and LCA operation classifiers are trained
  on gold annotations.

  \subsection{Global Inference for Expression Trees} 
  \label{method:inf}

    In this subsection, we define the scoring functions corresponding to
    the decomposed problems, and show how we combine these scores to
    perform global inference.  For a problem $P$ with quantities $q_1,
    q_2, \ldots, q_n$, we define the following scoring functions:
    
    \begin{enumerate}

      \item $\Pair(q_i, q_j, op)$ : Scores the likelihood of
      $\Op_{LCA}(q_i,q_j,\T) = op$, where $\T$ is a monotone
      expression tree of the solution expression of $P$. A multiclass
      classifier trained to predict LCA operations
      (Section \ref{method:lca}) can provide these scores.

      \item $\Rel(q)$ : Scores the likelihood of quantity $q$ being an
      irrelevant quantity in $P$, that is, $q$ is not used in creating
      the solution. A binary classifier trained to predict whether a
      quantity $q$ is relevant or not
      (Section \ref{method:relevance}), can provide these scores.

    \end{enumerate}

    For an expression $E$, let $\I(E)$ be the set of all quantities in
    $P$ which are not used in expression $E$. Let $\T$ be a monotonic
    expression tree for $E$.  We define $\Score(E)$ of an expression
    $E$ in terms of the above scoring functions and a scaling
    parameter $w_{\Rel}$ as follows:
    \begin{align}
      \Score(E) = &
      w_{\Rel} \sum_{q \in \I(E)} \Rel(q) +   \label{eq:inf1}  \\
       & \sum_{q_i,q_j \notin \I(E)} \Pair(q_i, q_j, \Op_{LCA}(q_i, q_j, \T))
       \notag
    \end{align}

    Our final expression tree is an outcome of a constrained
    optimization process, following ~\cite{RothYi04,ChangRaRo12}. Our
    objective function makes use of the scores returned by
    $\Rel(\cdot)$ and $\Pair(\cdot)$ to determine the expression tree
    and is constrained by legitimacy and background knowledge
    constraints, detailed below.

    \begin{enumerate}

      \item \textbf{Positive Answer}: Most arithmetic problems asking
      for amounts or number of objects usually have a positive number
      as an answer. Therefore, while searching for the best scoring
      expression, we reject expressions generating negative answer.

      \item \textbf{Integral Answer}: Problems with questions such as
      `how many'' usually expect integral solutions. We only consider
      integral solutions as legitimate outputs for such problems.

    \end{enumerate}

    Let $\mathcal{C}$ be the set of valid expressions that can be
    formed using the quantities in a problem $P$, and which satisfy
    the above constraints. The inference algorithm now becomes the
    following: \begin{align} \arg\max_{E \in \mathcal{C}} \Score(E) \end{align}

    The space of possible expressions is large, and we employ a beam
    search strategy to find the highest scoring constraint satisfying
    expression~\cite{ChangRaRo12}. We construct an expression tree
    using a bottom up approach, first enumerating all possible sets of
    irrelevant quantities, and next over all possible expressions,
    keeping the top $k$ at each step. We give details below.

    \begin{enumerate}

      \item \textbf{Enumerating Irrelevant Quantities}: We generate a
      state for all possible sets of irrelevant quantities, ensuring
      that there is at least two relevant quantities in each state. We
      refer to each of the relevant quantities in each state as a
      term. Therefore, each state can be represented as a set of
      terms.

      \item \textbf{Enumerating Expressions}: For generating a next
      state $S'$ from $S$, we choose a pair of terms $t_i$ and $t_j$
      in $S$ and one of the four basic operations, and form a new term
      by combining terms $t_i$ and $t_j$ with the operation. Since we
      do not know which of the possible next states will lead to the
      optimal goal state, we enumerate all possible next states (that
      is, enumerate all possible pairs of terms and all possible
      operations); we prune the beam to keep only the top $k$
      candidates. We terminate when all the states in the beam have
      exactly one term.

    \end{enumerate}

    Once we have a top $k$ list of candidate expression trees, we
    choose the highest scoring tree which satisfies the
    constraints. However, there might not be any tree in the beam which
    satisfies the constraints, in which case, we choose the top
    candidate in the beam. We use $k=200$ in our experiments. 
    
    In order to choose the value for the $w_{\Rel}$, we search over
    the set $\{ 10^{-6}, 10^{-4}, 10^{-2}, 1, 10^2, 10^4, 10^6 \}$,
    and choose the parameter setting which gives the highest accuracy
    on the training data.

  \subsection{Quantity Schema} 
  \label{method:quantschema}  

  In order to generalize across problem types as well as over simple
  manipulations of the text, it is necessary to train our system only
  with relevant information from the problem text. E.g., for the
  problem in example $2$, we do not want to take decisions based on
  how Tom earned money. Therefore, there is a need to extract the
  relevant information from the problem text. To this end, we
  introduce the concept of a {\em quantity schema} which we extract
  for each quantity in the problem's text. Along with the question
  asked, the quantity schemas provides all the information needed to
  solve most arithmetic problems.

  A quantity schema for a quantity $q$ in problem $P$ consists of the
  following components.

  \begin{enumerate}

    \item \textbf{Associated Verb} For each quantity $q$, we detect
      the verb associated with it. We traverse up the dependency tree
      starting from the quantity mention, and choose the first verb
      we reach. We used the easy first dependency parser~\cite{GoldbergEl10}.

    \item \textbf{Subject of Associated Verb} We detect the noun
      phrase, which acts as subject of the associated verb (if one exists).

    \item \textbf{Unit} We use a shallow parser to detect the phrase
      $p$ in which the quantity $q$ is mentioned. All tokens of the
      phrase (other than the number itself) are considered as unit
      tokens. Also, if $p$ is followed by the prepositional phrase
      ``of'' and a noun phrase (according to the shallow parser
      annotations), we also consider tokens from this second noun
      phrase as unit tokens. Finally, if no unit token can be extracted,
      we assign the unit of the neighboring quantities as the unit of
      $q$ (following previous work \cite{HosseiniHaEt14}).

    \item \textbf{Related Noun Phrases} We consider all noun phrases which
      are connected to the phrase $p$ containing quantity $q$, with NP-PP-NP
      attachment. If only one quantity is mentioned in a sentence, we
      consider all noun phrases in it as related.

    \item \textbf{Rate} We determine whether quantity $q$ refers to a
      {\em rate} in the text, as well as extract two unit components
      defining the rate. For example, ``7 kilometers per hour'' has
      two components ``kilometers'' and ``hour''. Similarly, for
      sentences describing unit cost like ``Each egg costs 2
      dollars'', ``2'' is a rate, with units ``dollars'' and ``egg''.

  \end{enumerate}

  In addition to extracting the quantity schemas for each quantity, we
  extract the surface form text which poses the question. For example,
  in the question sentence, ``How much will John have to pay if he
  wants to buy 7 oranges?'', our extractor outputs ``How much will
  John have to pay'' as the question.

  \subsection{Relevance Classifier}
  \label{method:relevance}

    We train a binary SVM classifier to determine, given problem text
    $P$ and a quantity $q$ in it, whether $q$ is needed
    in the numeric expression generating the solution. We train on
    gold annotations and use the score of the classifier as the
    scoring function $\Rel(\cdot)$.

  \subsubsection{Features}

    The features are extracted from the quantity schemas and can
    be broadly categorized into three groups:

    \begin{enumerate}

      \item \textbf{Unit features}: Most questions specifically
      mention the object whose amount needs to be computed, and hence
      questions provide valuable clue as to which quantities can be
      irrelevant.  We add a feature for whether the unit of quantity
      $q$ is present in the question tokens. Also, we add a feature
      based on whether the units of other quantities have better
      matches with question tokens (based on the number of tokens
      matched), and one based on the number of quantities which have
      the maximum number of matches with the question tokens.

      \item \textbf{Related NP features}: Often units are not enough
      to differentiate between relevant and irrelevant quantities.
      Consider the following:

      \begin{table}[H]
      \centering
      \small
      \begin{tabular}{|p{6cm}|}
        \hline Example 3 \\
        \hline Problem : {\em There are 8 apples in a pile on the
        desk. Each apple comes in a package of 11. 5 apples are added
        to the pile. How many apples are there in the pile?} \\
        \hline Solution : $(8+5) = 13$ \\
        \hline
      \end{tabular}
      \label{tab:example3}
      \end{table}

      The relevance decision depends on the noun phrase ``the pile'',
      which is absent in the second sentence. We add a feature
      indicating whether a related noun phrase is present in the
      question. Also, we add a feature based on whether the related
      noun phrases of other quantities have better match with the
      question. Extraction of related noun phrases is described in
      Section \ref{method:quantschema}.

      \item \textbf{Miscellaneous Features}: When a problem mentions
      only two quantities, both of them are usually relevant. Hence,
      we also add a feature based on the number of quantities
      mentioned in text.

    \end{enumerate}

    We include pairwise conjunction of the above features.

  \subsection{LCA Operation Classifier}
  \label{method:lca}  

    In order to predict LCA operations, we train a multiclass SVM
    classifier.  Given problem text $P$ and a pair of quantities $p_i$
    and $p_j$, the classifier predicts one of the six labels described
    in Eq.~\ref{expr:lca}.  We consider the confidence scores for each
    label supplied by the classifier as the scoring function
    $\Pair(\cdot)$.

  \subsubsection{Features}

    We use the following categories of features:

    \begin{enumerate}

      \item \textbf{Individual Quantity features}: Dependent verbs
      have been shown to play significant role in solving addition and
      subtraction problems \cite{HosseiniHaEt14}. Hence, we add the
      dependent verb of the quantity as a feature. Multiplication and
      division problems are largely dependent on rates described in
      text. To capture that, we add a feature based on whether the
      quantity is a rate, and whether any component of rate unit is
      present in the question. In addition to these quantity schema
      features, we add selected tokens from the neighborhood of the
      quantity mention. Neighborhood of quantities are often highly
      informative of LCA operations, for example, ``He got 80 more
      marbles'', the term ``more'' usually indicates addition. We add
      as features adverbs and comparative adjectives mentioned in a
      window of size $5$ around the quantity mention.

      \item \textbf{Quantity Pair features}: For a pair ($q_i, q_j$)
      we add features to indicate whether they have the same dependent
      verbs, to indicate whether both dependent verbs refer to the
      same verb mention, whether the units of $q_i$ and $q_j$ are the
      same and, if one of them is a rate, which component of the unit
      matches with the other quantity's unit. Finally, we add a
      feature indicating whether the value of $q_i$ is greater than
      the value of $q_j$.

      \item \textbf{Question Features}: Finally, we add a few features
      based on the question asked. In particular, for arithmetic
      problems where only one operation is needed, the question
      contains signals for the required operation. Specifically, we
      add indicator features based on whether the question mentions
      comparison-related tokens (e.g., ``more'', ``less'' or
      ``than''), or whether the question asks for a rate (indicated by
      tokens such as ``each'' or ``one'').

    \end{enumerate}

    We include pairwise conjunction of the above features. For both
    classifiers, we use the Illinois-SL
    package \footnote{
    http://cogcomp.cs.illinois.edu/page/software\_view/Illinois-SL}
    under default settings.

\section{Experimental Results}
\label{sec:results}

    \begin{table*}
    \centering \small  
    \begin{tabular}{|l||C{1.2cm}|C{1.2cm}||C{1.2cm}|C{1.2cm}||C{1.2cm}|C{1.2cm}||}
      \hline
      & \multicolumn{2}{c||}{AI2} & \multicolumn{2}{c||}{IL} 
      & \multicolumn{2}{c||}{CC} \\
      \cline{2-7}
      & Relax & Strict & Relax & Strict & Relax & Strict \\ \hline\hline    
      All features & \textbf{88.7} & \textbf{85.1} & \textbf{75.7} & \textbf{75.7} 
      & 60.0 & 25.8 \\ \hline
      No Individual Quantity features & 73.6 & 67.6 & 52.0 & 52.0 & 29.2 & 0.0 
      \\ \hline
      No Quantity Pair features & 83.2 & 79.8 & 63.6 & 63.6 & 49.3 & 16.5\\ \hline
      No Question features & 86.8 & 83.9 & 73.3 & 73.3 & \textbf{60.5} & 
      \textbf{28.3}\\ \hline
    \end{tabular}  
    \caption{\footnotesize Performance of LCA Operation classifier
    on the datasets AI2, IL and CC.}\label{table:lca}
    \end{table*}

  In this section, we evaluate the proposed method on publicly
  available datasets of arithmetic word problems. We evaluate
  separately the relevance and LCA operation classifiers, and show the
  contribution of various features. Lastly, we evaluate the
  performance of the full system, and quantify the gains achieved by
  the constraints.

  \subsection{Datasets}
    
    We evaluate our system on three datasets, each of which
    comprise a different category of arithmetic word problems.

    \begin{enumerate}

      \item \textbf{AI2 Dataset}: This is a collection of $395$
        addition and subtraction problems, released by
        \cite{HosseiniHaEt14}. They performed a $3$-fold cross
        validation, with every fold containing problems from different
        sources. This helped them evaluate robustness to domain
        diversity. We follow the same evaluation setting.

      \item \textbf{IL Dataset}: This is a collection of arithmetic
        problems released by \cite{RoyViRo15}. Each of these problems
        can be solved by performing one operation. However, there are
        multiple problems having the same template. To counter this,
        we perform a few modifications to the dataset. First, for each
        problem, we replace the numbers and nouns with the part of
        speech tags, and then we cluster the problems based on
        unigrams and bigrams from this modified problem text. In
        particular, we cluster problems together whose unigram-bigram
        similarity is over $90$\%. We next prune each cluster to keep
        at most $5$ problems in each cluster. Finally we create the
        folds ensuring all problems in a cluster are assigned to the
        same fold, and each fold has similar distribution of all
        operations. We have a final set of $562$ problems, and we use
        a $5$-fold cross validation to evaluate on this dataset.

      \item \textbf{Commoncore Dataset}: In order to test our system's
        ability to handle multi-step problems, we create a new dataset
        of multi-step arithmetic problems. The problems were extracted from 
        {\em www.commoncoresheets.com}. In total, there were 
        $600$ problems, $100$ for each of the following types:
        
        \begin{enumerate}
          \item Addition followed by Subtraction
          \item Subtraction followed by Addition
          \item Addition and Multiplication
          \item Addition and Division
          \item Subtraction and Multiplication
          \item Subtraction and Division
        \end{enumerate}    

        This dataset had no irrelevant quantities. Therefore, we did not
        use the relevance classifier in our evaluations.

        In order to test our system's ability to generalize across
        problem types, we perform a $6$-fold cross validation, with
        each fold containing all the problems from one of the
        aforementioned categories. This is a more challenging setting
        relative to the individual data sets mentioned above, since we
        are evaluating on multi-step problems, without ever looking at
        problems which require the same set of operations.

    \end{enumerate}

  \subsection{Relevance Classifier}
    
    Table \ref{table:rel} evaluates the performance of the relevance
    classifier on the AI2 and IL datasets. We report two accuracy
    values: Relax - fraction of quantities which the classifier got
    correct, and Strict - fraction of math problems, for which all
    quantities were correctly classified. We report accuracy using all
    features and then removing each feature group, one at a time.
    
    \begin{table}[H]
    \centering \small  
    \begin{tabular}{|l||C{0.75cm}|C{0.75cm}||C{0.75cm}|C{0.75cm}||}
      \hline
      & \multicolumn{2}{c||}{AI2} & \multicolumn{2}{c||}{IL} \\
      \cline{2-5}
      & Relax & Strict & Relax & Strict \\ \hline \hline   
      All features & 94.7 & 89.1 & \textbf{95.4} & \textbf{93.2} \\ \hline
      No Unit features & 88.9 & 71.5 & 92.8 & 91.0 \\ \hline
      No NP features & \textbf{94.9} & \textbf{89.6} & 95.0 & 91.2 \\ \hline
      No Misc. features & 92.0 & 85.9 & 93.7 & 89.8 \\ \hline
    \end{tabular}  
    \caption{\footnotesize Performance of Relevance classifier on the
    datasets AI2 and IL.} \label{table:rel}
    \end{table}  
  
    We see that features related to units of quantities play the
    most significant role in determining relevance of quantities. Also,
    the related NP features are not helpful for the AI2 dataset.

  \subsection{LCA Operation Classifier}

    Table \ref{table:lca} evaluates the performance of the LCA
    Operation classifier on the AI2, IL and CC datasets. As before, we
    report two accuracies - Relax - fraction of quantity pairs for
    which the classifier correctly predicted the LCA operation, and
    Strict - fraction of math problems, for which all quantity pairs
    were correctly classified. We report accuracy using all features
    and then removing each feature group, one at a time. 
    
    The strict and relaxed accuracies for IL dataset are identical,
    since each problem in IL dataset only requires one operation. The
    features related to individual quantities are most significant;
    in particular, the accuracy goes to $0.0$ in the CC dataset, without
    using individual quantity features. The question features are not
    helpful for classification in the CC dataset. This can be
    attributed to the fact that all problems in CC dataset require
    multiple operations, and questions in multi-step problems usually
    do not contain information for each of the required operations.
    
  \subsection{Global Inference Module}

    Table \ref{table:full} shows the performance of our system in
    correctly solving arithmetic word problems. We show the impact of
    various contraints, and also compare against previously best known
    results on the AI2 and IL datasets. We also show results using 
    each of the two constraints separately, and using no constraints 
    at all.
    
    \begin{table}[H]
    \begin{center}
    \begin{tabular}{|l||c|c|c|}
      \hline
      & AI2 & IL & CC \\ \hline\hline
      All constraints & 72.0 & \textbf{73.9} & \textbf{45.2} \\ \hline
      Positive constraint & \textbf{78.0} & 72.5 & 36.5 \\ \hline
      Integral constraint & 71.8 & 73.4 & 39.0 \\ \hline
      No constraint & 77.7 & 71.9 & 29.6 \\ \hline\hline
      \cite{HosseiniHaEt14} & 77.7 & - & -\\ \hline
      \cite{RoyViRo15} & - & 52.7 & - \\ \hline
      \cite{KushmanZeBa14} & 64.0 & 73.7 & 2.3 \\ \hline
    \end{tabular}  
    \end{center} 
    \caption{\footnotesize Accuracy in correctly solving arithmetic problems.
    First four rows represent various configurations of our system. We achieve
    state of the art results in both AI2 and IL datasets.}\label{table:full}
    \end{table}

    The previously known best result in the AI2 dataset is reported in
    \cite{HosseiniHaEt14}. Since we follow the exact same evaluation
    settings, our results are directly comparable.  We achieve state
    of the art results, without having access to any additional
    annotated data, unlike \cite{HosseiniHaEt14}, who use labeled data
    for verb categorization.  For the IL dataset, we acquired the
    system of \cite{RoyViRo15} from the authors, and ran it with the
    same fold information. We outperform their system by an absolute
    gain of over $20 \%$. We believe that the improvement was mainly
    due to the dependence of the system of \cite{RoyViRo15} on lexical
    and neighborhood of quantity features. In contrast, features from
    quantity schemas help us generalize across problem types. Finally,
    we also compare against the template based system of
    \cite{KushmanZeBa14}.  \cite{HosseiniHaEt14} mentions the result
    of running the system of \cite{KushmanZeBa14} on AI2 dataset, and
    we report their result here.  For IL and CC datasets, we used the
    system released by \cite{KushmanZeBa14}.

    The integrality constraint is particularly helpful when division
    is involved, since it can lead to fractional answers. It does not
    help in case of the AI2 dataset, which involves only addition and
    subtraction problems. The role of the constraints becomes more
    significant in case of multi-step problems and, in particular,
    they contribute an absolute improvement of over $15\%$ over the
    system without constraints on the CC dataset. The template based
    system of \cite{KushmanZeBa14} performs on par with our system on
    the IL dataset. We believe that it is due to the small number of
    equation templates in the IL dataset. It performs poorly on the CC
    dataset, since we evaluate on unseen problem types, which do not
    ensure that equation templates in the test data will be seen in
    the training data.

  \subsection{Discussion}

    The leading source of errors for the classifiers are erroneous
    quantity schema extraction and lack of understanding of unknown or
    rare verbs. For the relevance classifier on the AI2 dataset, $25\%$
    of the errors were due to mistakes in extracting the quantity
    schemas and $20\%$ could be attributed to rare verbs. For the LCA
    operation classifier on the same dataset, $16\%$ of the errors were
    due to unknown verbs and $15\%$ were due to mistakes in extracting
    the schemas. The erroneous extraction of accurate quantity schemas
    is very significant for the IL dataset, contributing $57\%$ of the
    errors for the relevance classifier and $39\%$ of the errors for the
    LCA operation classifier.  For the operation classifier on the CC
    dataset, $8\%$ of the errors were due to verbs and $16\%$ were due to
    faulty quantity schema extraction.  Quantity Schema extraction is
    challenging due to parsing issues as well as some non-standard
    rate patterns, and it will be one of the future work targets.  For
    example, in the sentence, ``How many 4-dollar toys can he buy?'', we
    fail to extract the rate component of the quantity 4.

\section{Conclusion}
\label{sec:conc}

  This paper presents a novel method for understanding and solving a
  general class of arithmetic word problems. Our approach can solve
  all problems whose solution can be expressed by a read-once
  arithmetic expression, where each quantity from the problem text
  appears at most once in the expression.  We develop a novel
  theoretical framework, centered around the notion of monotone
  expression trees, and showed how this representation can be used to
  get a unique decomposition of the problem. This theory naturally
  leads to a computational solution that we have shown to uniquely
  determine the solution –- determine the arithmetic operation between
  any two quantities identified in the text. This theory underlies our
  algorithmic solution –- we develop classifiers and a constrained
  inference approach that exploits redundancy in the information, and
  show that this yields strong performance on several benchmark
  collections.  In particular, our approach achieves state of the art
  performance on two publicly available arithmetic problem datasets
  and can support natural generalizations. Specifically, our approach
  performs competitively on multistep problems, even when it has never
  observed the particular problem type before.
 
  Although we develop and use the notion of expression trees in the
  context of numerical expressions, the concept is more general. In
  particular, if we allow leaves of expression trees to represent
  variables, we can express algebraic expressions and equations in
  this framework. Hence a similar approach can be targeted towards
  algebra word problems, a direction we wish to investigate in the
  future.

  The datasets used in the paper are available for download at
  {\fontsize{7.6}{8} 
  {\em
  http://cogcomp.cs.illinois.edu/page/resource\_view/98}.

\section*{Acknowledgments}
This research was sponsored by DARPA (under agreement number
FA8750-13-2-0008), and a grant from AI2.  Any opinions, findings,
conclusions or recommendations are those of the authors and do not
necessarily reflect the view of the agencies.

\bibliographystyle{acl}
\bibliography{arithmetic,ccg,cited}

\end{document}